\documentclass[11pt]{article}

\usepackage{authordate1-4}
\usepackage{FRdina4}
\usepackage{amssymb}  
\usepackage{MyMacros}
\usepackage{color}
\usepackage{url}

\begin{document}
\pagestyle{headings}

\title{\textbf{Inverses, Conditionals and Compositional Operators in Separative Valuation Algebras}}

\author{Juerg Kohlas \\
\small Department of Informatics DIUF \\ 
\small University of Fribourg \\ 
\small CH -- 1700 Fribourg (Switzerland) \\ 
\small E-mail: \texttt{juerg.kohlas@unifr.ch} \\
\small \texttt{http://diuf.unifr.ch/drupal/tns/juerg\_kohlas
}
}
\date{\today}

\maketitle

%%%%%%%%%%%%%%%%%%%%%%%%%%%%%%%%%%%%%%%%%%%%%%%%%%%%%%%%%%%%%%%%%%%%%%%%%%%

\begin{abstract}
Compositional models were introduce by Jirousek and Shenoy in the general framework of valuation-based systems. They based their theory on an axiomatic system of valuations involving not only the operations of combination and marginalisation, but also of removal. They claimed that this systems covers besides the classical case of discrete probability distributions, also the cases of Gaussian densities and belief functions, and many other systems. 

Whereas their results on the compositional operator are correct, the axiomatic basis is not sufficient to cover the examples claimed above. We propose here a different axiomatic system of valuation algebras, which permits a rigorous mathematical theory of compositional operators in valuation-based systems and covers all the examples mentioned above. It extends the classical theory of inverses in semigroup theory and places thereby the present theory into its proper mathematical frame. Also this theory sheds light on the different structures of valuation-based systems, like regular algebras (represented by probability potentials), canncellative algebras (Gaussian potentials) and general separative algebras (density functions).
\end{abstract}

%%%%%%%%%%%%%%%%%%%%%%%%%%%%%%%%%%%%%%%%%%%%%%%%%%%%%%%%%%%%%%%%%%%%%%%%%%%
\tableofcontents

%%%%%%%%%%%%%%%%%%%%%%%%%%%%%%%%%%%%%%%%%%%%%%%%%%%%%%%%%%%%

\section{Introduction}

\cite{jirousekshenoy14,jirousekshenoy15} introduced compositional models in the general framework of valuation-based systems. They based their theory on an axiomatic system of valuations involving not only the operations of combination and marginalisation, but also of removal. They claimed that this systems covers besides the classical case of discrete probability distributions, also the cases of Gaussian densities and belief functions, and many other systems. 

Whereas their results on the compositional operator are correct, the axiomatic basis is not sufficient to cover the examples claimed above. We propose here a different axiomatic system of valuation algebras, which permits a rigorous mathematical theory of compositional operators in valuation-based systems and covers all the examples mentioned above. It extends the classical theory of inverses in semigroup theory and places thereby the present theory into its proper mathematical frame. Also this theory sheds light on the different structures of valuation-based systems, like regular algebras (represented by probability potentials), canncellative algebras (Gaussian potentials) and general separative algebras (density functions).

%%%%%%%%%%%%%%%%%%%%%%%%%%%%%%%%%%%%%%%%%%%%%%%%%%%%%%%%%%%%

\section{Separative Valuation Algebras}

In this section we briefly review the basic concepts of valuation algebras and in particular separative ones. Valuation-based systems were introduced in \cite{shenoyshafer90}. In \cite{kohlas03} the algebraic theory of the valuation algebras, algebraic structures based on valuation-based systems were defined and their algebraic theory to some extend developed. In particular, separative valuation algebras were introduced, algebras, which permit the removal of information. This is the basis of the present section.

In valuation-based systems, valuations represent information relative to some subsets of variables. These subsets form a lattice. The theory can as well developed for more general domains forming any \textit{lattice}. This covers then valuation-based systems whose elements provide information to partitions of some universe or families of compatible frames. So, let $(D,\leq)$ be a lattice whose elements are called \textit{domains} and denoted by lower-case letters like $x,y,z,\ldots$. Any pair $x,y$ of elements of $D$ has a least upper bound, denoted by $x \vee y$ and called the join of $x$ and $y$. The pair has also a greatest lower bound, denoted by $x \wedge y$ and called the meet of $x$ and $y$. In the case of subsets of variables, the lattice $(D;\leq)$ is distributive. But this need not be the case in general. For instance if $(D;\leq)$ is a lattice of partitions of some universe, it is no more distributive. For more about lattices we refer to \cite{daveypriestley97}

Let $\Psi$ denote a set whose elements are called \textit{valuations}. Elements of $\Psi$ are denoted by lower-case Greek letters like $\phi,\psi,\ldots$. Each valuation $\psi$ is associated with a domain in $D$ denoted by $d(\psi)$. Further, valuations can be combined and projected to lower domains. Thus, formally, we consider the following operations:
\begin{enumerate}
\item \textit{Labeling:} $d : \Psi \rightarrow D$, $\psi \mapsto d(\psi)$.
\item \textit{Combination:} $\cdot : \Phi \times \Phi \rightarrow \Phi$, $(\phi,\psi) \mapsto \phi \cdot \psi$.
\item \textit{Projection:} $\pi : \Phi \times D \rightarrow \Phi$, $(\psi,x) \mapsto \pi_x(\psi)$, defined only for $x \leq \psi$.
\end{enumerate}
These operations are subjected to the following axioms:
\begin{description}
\item[A1] \textit{Lattice:} $(D;\leq)$ is a lattice.
\item[A2]  \textit{Semiproup:} $(\Psi,\cdot)$ is a commutative semigroup.
\item[A3] \textit{Labeling:} $d(\phi \cdot \psi) = d(\phi) \vee d(\psi)$, $d(\pi_x(\psi)) = x$.
\item[A4] \textit{Projection:} If $x \leq y \leq d(\psi)$, then $\pi_x(\pi_y(\psi)) = \pi_x(\psi)$.
\item[A5] \textit{Combination:} If $d(\phi) = x$, $d(\psi) = y$, then $\pi_x(\phi \cdot \psi) = \phi \cdot \pi_{x \wedge y}(\psi)$.
\end{description}
This corresponds to the requirements of a valuation-based system as expressed in \cite{shenoyshafer90,jirousekshenoy14}. A system with signature $(\Psi,D;\leq,\wedge,\vee,\cdot,\pi)$ satisfying these axioms is called a \textit{valuation algebra}. Sometimes there are \textit{unit} elements $1_x$ in the semigroups of all elements with domain $x$ for  for all domains $x \in D$. These unit elements are assumed to satisfy:

\begin{description}
\item[A6] \textit{Units:} $\psi \cdot 1_x = \psi$, if $d(\psi) = x$, and $1_x \cdot 1_y = 1_{x \vee y}$.
\end{description}

In still other cases there are for all $x \in D$ null elements in the semigroups of all valuations with domain $x$. Null elements are assumed to satisfy

\begin{description}
\item[A7]  \textit{Null Elements:} $\psi \cdot 0_x = 0_x$ if $d(\psi) = x$, and if $x \leq y$, $d(\psi) = y$, then $\pi_x(\psi) = 0_x$ if and only if $\psi = 0_y$.
\end{description}

In some cases, a stronger verison of Axiom A5 holds:
\begin{description}
\item[A5'] \textit{Strong Combination:} If $d(\phi) = x$, $d(\psi) = y$ and $x \leq z \leq x \vee y$, then $\pi_z(\phi \cdot \psi) = \phi \cdot \pi_{y \wedge z}(\psi)$.
\end{description}
If the valuation algebra has units which satisfy Axiom A6, and if the lattice $(D;\leq)$ is modular, then A5' follows. But there are examples without units (for instance densities, see Section Ê\ref{sec:Exampl}), which satisfy A5'.

In \cite{shenoy94c,jirousekshenoy14} a removal operator was introduced, a kind of inverse to the combination operation. We propose here a mathematically more rigorous approach based on semigroup theory, which places the theory in the proper mathematical context and which serves also to clarify the algebraic structure of valuation  algebras with removal or division. A commutative semigroup like $(\Psi;\cdot)$ is called \textit{separative}, if $\phi \cdot \phi = \psi \cdot \psi = \phi \cdot \psi$ implies that $\phi = \psi$. This condition is necessary and sufficient to embed the semigroup $(\Psi;\cdot)$ into a semigroup $(\Psi^0;\cdot)$ which is a union of disjoint \textit{groups} $\delta(\psi)$,
\begin{eqnarray*}
\Psi^0 = \bigcup_{\psi \in \Psi} \delta(\psi).
\end{eqnarray*}
\cite{hewittzucker56}. We may consider $\Psi$ as a subset, a subsemigroup, of $\Psi^0$. Then any valuation $\psi$ in $\Psi$ has an \textit{inverse} $\psi^{-1}$ such that $\psi \psi^{-1} = f_\psi$, where $f_\psi$ is the unit in the group $\delta(\psi)$. In general, neither the inverses $\psi^{-1}$ nor the group units $f_\psi$ belong to $\Psi$, but only to $\Psi^0$. This will be illustrated by examples below (Section \ref{sec:Exampl}).

However, although separativity of the semigroup $(\Psi;\cdot)$ is necessary, it is not sufficient for the needs of division or removal in valuation algebras. Something a little bit stronger is needed, since division must be related to projection. This has been shown in \cite{kohlas03}. We summarise this theory here in a slightly more general framework. If $(\Psi;\cdot)$ is a separative semigroup,. we may define $\phi \equiv_\delta \psi$ if $\phi$ and $\psi$ belong to the same group, that is, if $\delta(\phi) = \delta(\psi)$. This is an equivalence relation on $\Psi$. Moreover, since $\phi \cdot \phi \in \delta(\phi)$, we have $\phi \cdot \phi \equiv_\delta \phi$. Further, the relation $\equiv_\delta$ is a semigroup congruence, that is, if $\phi \equiv_\delta \psi$ and $\eta$ is any other valuation, then also $\phi \cdot \eta \equiv_\delta \psi \cdot \eta$. This implies that the corresponding equivalence classes $[\psi]_\delta$ in $\psi$ are subsemigroups of $(\Psi;\cdot)$, since $\phi \equiv_\delta \psi$ implies $\phi \cdot \psi \equiv_\delta \phi \cdot \phi \equiv_\delta \phi$. Finally, these semigroups $[\psi]_\delta$ have the property that if $\phi,\psi,\eta \in [\psi]_\delta$, then $\eta \cdot \phi = \eta \cdot \psi$ implies $\phi = \psi$. Such semigroups are called \textit{cancellative} \cite{cliffordpreston67}.

Now, we have the ingredients to define \textit{separative valuation algebras}:

\begin{definition}
A valuation algebra $(\Psi,D;\leq,\wedge,\vee,\cdot,\pi)$ is called separative, if
\begin{description}
\item[S1] There is a combination congruence $\equiv_\delta$ in $\Psi$ such that for all $\psi \in \Psi$ and $x \leq d(\psi)$,
\begin{eqnarray*}
\psi \cdot \pi_x(\psi) \equiv_\delta \psi
\end{eqnarray*}
\item[S2] The semigroups $[\psi]_\delta$ are all cancellative.
\end{description}
\end{definition}

Note that S1 implies also $\psi \cdot \psi \equiv_\delta \psi$. For instance the multiplicative semigroup of positive integers is canellative. As is well known, it may be extended to the group of positive rational numbers, into which the positive integers are embedded. Now, this can be done for \textit{any} cancellative semigroup: Consider ordered pairs $(\psi,\phi)$ of elements of the same class $[\eta]_\delta$. Define for such pairs $(\phi,\psi) \equiv (\phi',\psi')$ if $\phi \cdot \psi' = \psi \cdot \phi'$. This is an equivalence relation among pairs $(\phi,\psi)$. Denote the corresponding equivalence classes by $[\phi,\psi]$. These classes represent the ``rational numbers'' or ``quotients'' extending the semigroup $[\psi]_\delta$ (see the next section for an illustration of this concept). In fact, let $\delta(\psi)$ denote the family of all equivalence classes $[\phi,\psi]$ of pairs in $[\psi]_\delta$. This is a group. The multiplication among elements of $\delta(\psi)$ is defined by
\begin{eqnarray}
[\phi,\psi] \cdot [\phi',\psi'] = [\phi \cdot \phi',\psi \cdot \psi'].
\end{eqnarray}
This operation is well defined, since the relation $\equiv$ is a multiplicative congruence. The unit is the class $[\psi,\psi]$ and the inverse of $[\phi,\psi]$ is $[\psi,\phi]$. The semigroup $[\psi]_\delta$ is embedded into the group by the one-to-one semigroup homomorphism $\psi \mapsto [\psi \cdot \psi,\psi]$. 

Let then
\begin{eqnarray*}
\Psi^0 = \bigcup_{\psi \in \Psi} \delta(\psi).
\end{eqnarray*}
be the union of the disjoint groups $\delta(\psi)$. This is a semigroup, where multiplication is defined as above, but this time between classes $[\phi,\psi]$ and $[\phi',\psi']$ belonging possibly to different groups. So, the semigroup $(\Psi^0;\cdot)$ is a union of disjoint groups and now the semigroup $(\Psi;\cdot)$ is embedded into it by the same map as above. The unit elements $f_\psi$ of the groups $\delta(\psi)$ are idempotent elements, $f_\psi \cdot f_\psi = f_\psi$, and they are closed under multiplication, that is $f_\phi \cdot f_\psi = f_{\phi \cdot \psi}$ is still idempotent and in fact the unit element of the group $\delta(\phi \cdot \psi)$. The idempotent elements $f_\psi$ form so an \textit{idempotent} subsemigroup $(F;\cdot)$ of $(\Psi^0;\cdot)$. It is well-known that in an idempotent semigroup a \textit{partial order} may be defined which determines a semilattice. So, we may define $f_\psi \leq f_\phi$ if $f_\psi \cdot f_\phi = f_\phi$. Under this order we have $f_\phi \cdot f_\psi = f_\phi \vee f_\psi$. This order may be carried over to the groups $\delta(\psi)$ by specifying $\delta(\psi) \leq \delta(\phi)$ if $f_\psi \leq f_\phi$. Then we have also $\delta(\phi \cdot \psi) = \delta(\phi) \vee \delta(\psi)$. So, $\Psi$ is embedded into a semigroup $\Psi^0$ which is a union of disjoint groups forming a join-semilattice.This order of groups is the exact mathematical counterpart of the domination relation introduced in \cite{jirousekshenoy14}. We remark that the first condition in the definition of a separative valuation algebra implies that 
\begin{eqnarray} \label{eq:DomProj}
\delta(\pi_x(\psi)) \leq \delta(\psi) \textrm{ for all}\ x \leq \d(\psi).
\end{eqnarray}
Note also that if $\delta(\psi) \leq \delta(\phi)$ then $\phi \cdot f_\psi  = \phi \cdot f_\phi \cdot f_\psi = \phi \cdot f_\phi = \phi$. 

Denote by $\Psi_x$ the set of all valuations with domain $x$, that is $d(\psi) = x$. It is a subsemigroup of $\Psi$ and
\begin{eqnarray*}
\Psi = \bigcup_{x \in D} \Psi_x.
\end{eqnarray*}
In some cases the semigroup $\Psi_x$ are already cancellative. Then we call the valuation algebra \textit{cancellative}. Every semigroup $\Psi_x$ is then embedded into a group $\delta(\psi)$ if $\psi \in \Psi_x$. This is for example that case of Gaussian densities (see next section). In other cases, all groups $\delta(\phi)$ belong entirely to $\Psi$, that is, $\Psi = \Psi^0$. Necessary and sufficient for this is that for all $\psi \in \Psi$ and all $x \leq d(\psi)$ there is a $\chi \in \Psi_x$ such that 
\begin{eqnarray*}
\psi = \psi \cdot \pi_x(\psi) \cdot \chi.
\end{eqnarray*}
Then the valuation algebra $\Psi$ is called \textit{regular}, since the semigroup $(\Psi,\cdot)$ is regular in the sense of semigroup theory \cite{cliffordpreston67}. An example of a regular valuation algebra is provided by probability potentials (see next section). 

In regular valuation algebras, for inverses as well as for idempotents, all projections do exist and belong to the algebra. This is not the case for general separative valuation algebras. We can however extend projection at least partially into the semigroup $(\Psi^0;\cdot)$. We assume here that the valuation algebra either has units satisfying Axiom A6 or else satisfies the strong combination Axiom A5'. If $d(\phi) = x$, $d(\psi) = y$ and $x \leq z \leq y \vee x$, then, in the first case,  by the Combination Axiom A5
\begin{eqnarray} \label{eq:ProjLemma}
\pi_z(\phi \cdot \psi) = \pi_z((\phi \cdot 1_z) \cdot \psi) = (\phi \cdot 1_z) \cdot \pi_{y \wedge z}(\psi) = \phi \cdot \pi_{y \wedge z}(\psi).
\end{eqnarray}
In the second case this result is Axiom A5'. We shall see that this enables us to extend projection beyond $\Psi$. 

First we extend labeling to $\Psi^0$. If $\eta \in \delta(\psi)$ for some $\psi \in \Psi$, then we define $d(\eta) = d(\psi)$. Clearly this is an extension of the labeling operation from $\Psi$ to $\Psi^0$. Consider $\eta \in \delta(\psi)$ and $\eta' \in \delta(\psi')$. Then $\eta \cdot \eta' \in \delta(\psi \cdot \psi') = \delta(\psi) \vee \delta(\psi')$. Hence it follows $d(\eta \cdot \eta') = d(\eta) \vee d(\eta')$. Therefore the Labeling Axiom A3 extends to all of $\Psi^0$.

Next we turn to projection. Remind that an element $\eta \in \Psi^0$ is an equivalence class $[\phi,\psi] = [\phi \cdot \phi,\phi] \cdot [\psi,\psi \cdot \psi] = \phi \cdot \psi^{-1}$. In many cases, $\psi$ is of the form $\psi' \cdot f$, where $f$ is an idempotent. This is for instance the case for conditionals, see Section \ref{sec:Conditionals} below. Then $d(\psi'),d(f) \leq d(\phi) = d(\psi)$ and also $\delta(\psi'),\delta(f) \leq \delta(\phi) = \delta(\psi)$. It follows that $\eta = \phi \cdot f \cdot \psi'^{-1} = \phi \cdot \psi'^{-1}$. Then we may define projection of $\eta$ for $d(\psi') \leq x \leq d(\phi)$ by
\begin{eqnarray} \label{eq:ExtProj}
\pi_x(\eta) = \pi_x(\phi) \cdot \psi'^{-1},
\end{eqnarray}
The representation of an element $\eta$ as $\phi \cdot \psi^{-1}$ is not unique. So, we must show that definition (\ref{eq:ExtProj}) is unambiguous. Therefore assume that $\eta = \phi \cdot \psi^{-1} = \phi' \cdot \psi'^{-1}$, where $d(\psi),d(\psi') \leq d(\phi) = d(\phi')$, $\delta(\eta) = \delta(\phi) = \delta(\phi')$ and $\delta(\psi),\delta(\psi') \leq \delta(\eta)$. We obtain then $\phi \cdot \psi' = \phi' \cdot \psi$. It follows for $x$ such that $d(\psi),d(\psi') \leq x \leq d(\eta)$, using (\ref{eq:ProjLemma})
\begin{eqnarray}
\pi_x(\phi) \cdot \psi' = \pi_x(\phi \cdot \psi') = \pi_x(\phi' \cdot \psi) = \pi_x(\phi') \cdot \psi,
\nonumber
\end{eqnarray}
since all elements involved belong to $\Psi$. Further, $\delta(\psi) \leq \delta(\phi)$ implies $\delta(\psi) \leq \delta(\pi_x(\phi))$ since $\delta(\phi) = \delta(\phi \cdot \psi)$, hence $\delta(\pi_x(\phi)) =  \delta(\pi_x(\phi \cdot \psi))= \delta(\pi_x(\phi) \cdot \psi) = \delta(\pi_x(\phi)) \vee \delta(\psi)$, and in the same way we conclude that $\delta(\psi') \leq \delta(\phi')$. From this we obtain
\begin{eqnarray}
\pi_x(\phi) \cdot \psi^{-1} = \pi_x(\phi') \cdot \psi'^{-1},
\nonumber
\end{eqnarray}
which shows that $\pi_x(\eta)$ is well defined for all $x \in D$ whenever there exist $\phi,\psi \inÊ\Psi$ such that $\eta = \phi \cdot \psi^{-1}$ with $d(\psi) \leq x \leq d(\phi)$ and $\delta(\psi) \leq \delta(\phi)$. Also, we have always $\pi_x(\eta) = \eta$ if $d(\eta) = x$. Finally, assume $\eta \in \Psi$, then for all $x \leq d(\eta)$, we have $\eta = \eta \cdot \pi_x(\eta) \cdot (\pi_x(\eta))^{-1}$. Then, by the new definition of projection of $\eta$ to $x$, $\pi_x(\eta \cdot \pi_x(\eta)) \cdot (\pi_x(\eta))^{-1} = \pi_x(\eta)$, where on the left is the projection as defined in $\Psi$. This shows that the new definition of projection is indeed an extension of the projection in $\Psi$.

It turns out that this extension of the projection operator still satisfies the Transitivity and Combination Axioms in $\Psi^0$.

\begin{theorem} \label{th:ExtProjTransComb}
If $(\Psi,D;\leq,d,\cdot,\pi)$ is a separative valuation algebra, satisfying either Axiom A6 or else Axiom A5', and $\pi : \Psi^0 \times D \rightarrow \Psi^0$ is the partially defined extension of projection to $\Psi^0$, then the following holds:
\begin{enumerate}
\item If $\pi_x(\eta)$ exists for $x \leq d(\eta)$, $\eta \in \Psi^0$ and $x \leq y \leq d(\eta)$, then
\begin{eqnarray}
\pi_x(\pi_y(\eta)) = \pi_x(\eta).
\end{eqnarray}
\item If $\eta_1,\eta_2 \in \Psi^0$ with $d(\eta_1) = x$, $d(\eta_2) = y$ and if $\pi_{x \wedge y}(\eta_2)$ exists, then $\pi_x(\eta_1 \cdot \eta_2)$ exists and
\begin{eqnarray}
\pi_x(\eta_1 \cdot \eta_2) = \eta_1 \cdot \pi_{x \wedge y}(\eta_2).
\end{eqnarray}
\end{enumerate}
\end{theorem}

\begin{proof}
1.) Assume that $\eta = \phi \cdot \psi^{-1}$ with $d(\psi) \leq x \leq y \leq d(\eta) = d(\phi)$ and $\delta(\psi) \leq \delta(\phi)$. Then it follows as above that $\delta(\psi) \leq \delta(\pi_y(\phi))$. So, $\pi_y(\eta)$ exists too. In $\Psi$ we have $\pi_x(\phi) = \pi_x(\pi_y(\phi))$. Therefore $\pi_x(\pi_{y}(\eta)) = \pi_x(\pi_y(\phi)) \cdot \psi^{-1}) = \pi_x(\pi_y(\phi)) \cdot \psi^{-1} = \pi_x(\phi) \cdot \psi^{-1} = \pi_x(\eta)$.

2.) Assume $\eta_1 = \phi_1 \cdot \psi_1^{-1}$ and $\eta_2 = \phi_2 \cdot \psi_2^{-1}$, where $d(\psi_1) \leq d(\phi_1) = x$, $\delta(\psi_1) \leq \delta(\phi_1)$ and $d(\psi_2) \leq x \wedge y$ , $\delta(\psi_2) \leq \delta(\phi_2)$. Then
\begin{eqnarray}
\eta_1 \cdot \eta_2 = (\phi_1 \cdot \phi_2) \cdot (\psi_1 \cdot \psi_2)^{-1},
\nonumber
\end{eqnarray}
where $d(\psi_1 \cdot \psi_2) = d(\psi_1) \vee d(\psi_2) \leq x \leq d(\phi_1) \vee d(\phi_2)$ and $\delta(\psi_1 \cdot \psi_2) = \delta(\psi_1) \vee \delta(\psi_2) \leq \delta(\phi_1) \vee \delta(\phi_2) = \delta(\phi_1 \cdot \phi_2)$. Then we have by the Combination Axiom in $\Psi$,
\begin{eqnarray}
\pi_x(\eta_1 \cdot \eta_2) &=& \pi_x(\phi_1 \cdot \phi_2) \cdot (\psi_1 \cdot \psi_2)^{-1}
\nonumber \\
&=&\phi_1 \cdot \pi_{x \wedge y}(\phi_2) \cdot (\psi_1 \cdot \psi_2)^{-1}
\nonumber \\
&=&(\phi_1 \cdot \psi_1^{-1}) \cdot (\pi_{x \wedge y}(\phi_2) \cdot \psi_2^{-1})
\nonumber \\
&=&\eta_1 \cdot \pi_{x \wedge y}(\eta_2).
\nonumber
\end{eqnarray}
So, $\pi_x(\phi_1 \cdot \phi_2)$ exists and the combination axiom holds under these circumstances.
\end{proof}
This theory of partial projection is essential for generalising the formalism of conditional probability distributions to separative valuation algebras, and for introducing compositional operators, see Sections \ref{sec:Conditionals} and \ref{sec:CompOp}. Before we turn to these subjects, we illustrate the theory of separative valuation algebras by some examples.

%%%%%%%%%%%%%%%%%%%%%%%%%%%%%%%%%%%%%%%%%%%%%%%%%%%%%%%%%%%%

\section{Examples of Separative Valuation Algebras} \label{sec:Exampl}

Here we present four examples of separative valuation algebras, where one of them is regular, another one cancellative and the remaining two neither regular nor cancellative. These examples should allow to better understand the abstract structure of a semigroup which is the union of disjoint groups and division as presented in the previous section.

We begin with a regular valuation algebra.

\begin{example} \textit{Probability Potentials:}

Let $(D;\leq)$ be the (distributive) lattice of finite subsets of a countable family of variables $X_i$, $i=1,2,\ldots$. Let $\Theta_i$ denote finite sets of possible values for the variables $X_i$. For a finite subset $s$ in $D$ define $\Theta_s$ to be the Cartesian product of the sets $\Theta_i$ for $i \in s$. The elements of $\Theta_s$ are $s$-tuples $x$ with components $x_i \in \Theta_i$ for $i \in s$. If $x$ is an $s$-tuple and $t \subseteq s$, then $x_t$ denotes the subtuple of $x$ of components $x_i$ with $i \in t$. For any finite set $s$ of variables, we consider functions $p : \Theta_s \rightarrow \mathbb{R}^+ \cup \{0\}$ of $s$-tuples into nonnegative real numbers. Let $\Psi_s$ be the set of all such functions on a set $s$ of variables and define
\begin{eqnarray*}
\Psi = \bigcup_{s \in D} \Psi_s.
\end{eqnarray*}
The functions $p$ may be normalised to the sum $1$, and represent then discrete probability distributions on the set $\Theta_s$. Therefore, we call the functions $p$ \textit{probability potentials}. Now, we define the operations of a valuation algebras as follows:
\begin{enumerate}
\item \textit{Labeling:} $d(p) = s$ if $p$ is a probability potential on $\Theta_s$.
\item \textit{Combination:} For probability potentials $p$ and $q$ with $d(p) = s$ and $d(q) = t$ and $x \in \Theta_{s \cup t}$,
\begin{eqnarray}
(p \cdot q)(x) = p(x_s)q(x_t).
\nonumber
\end{eqnarray}
\item \textit{Projection:} For a probability potential $p$ with $d(p) = s$, $t \subseteq s$, and $y \in \Theta_t$,
\begin{eqnarray}
(\pi_t(p))(y) = \sum_{x: x_t=y} p(x).
\nonumber
\end{eqnarray}
\end{enumerate}
This is an instance of a semiring-valuation algebra, namely for the semiring of nonnegative real numbers, see \cite{kohlaswilson08}, where many more examples of separative semiring-valuation algebras may be found. Let $supp(p) = \{x \in \Theta_s:p(x) > 0\}$ be the \textit{support} of a probability potential $p$ with $d(p) = s$. This valuation algebra of probability potential is \textit{regular}; any quotient $p(x)/q(x)$ of two probability potential with the same support sets $supp(p) = supp(q)$ is still a probability potential (assuming to vanish outside $supp(p)$). Hence we have $\Psi = \Psi^0$. The inverse of a probability potential $p$ is $p^{-1}(x) = 1/p(x)$ for $x \in supp(p)$ and equal to zero elsewhere. The idempotents in groups $\delta(p)$ are the functions $f_p(x) = 1$ if $x \in supp(p)$ and vanishing elsewhere on $\Theta_s$. The groups $\delta(p)$ are formed by the potentials with the same support $supp(p)$ and the lattice of theses groups is closely related to the lattice of the support sets $supp(p) \subseteq \Theta_s$. The valuation algebra of probability potentials is at the base of local computation schemes for Bayesian or more general probabilistic networks. As a regular valuation algebra it allows for local computation with division \cite{lauritzenjensen97,kohlas03}.
\end{example}

The next example is also related to probability theory, it is however only separative, but neither regular nor cancellative.

\begin{example} \textit{Density Functions:}

This example is based on domains formed by sets of real-valued variables. Let $(D;\subseteq)$ be the lattice of finite subsets of $\omega = \{1,2\ldots\}$. We consider here the linear vector spaces $\mathbb{R}^s$ of real valued tuples $x : s \rightarrow \mathbb{R}$ where $s$ is a finite subset of $\omega$. On a space $\mathbb{R}^s$ we consider nonnegative functions $f = \mathbb{R}^s \rightarrow \mathbb{R}^+ \cup \{0\}$, whose integrals 
\begin{eqnarray}
\int_{- \infty}^{+ \infty} f(x) dx
\end{eqnarray}
exist and are finite. To simplify, we consider continuous functions and Rieman integrals; it would also be possible to consider measurable functions and Lebesgue integrals \cite{kohlas03}. Such functions are called \textit{density functions} or shortly \textit{densities} on $\mathbb{R}^s$.  By normalisation a density can be become a continuous density function in the sense of probability theory. We define the operations of a valuation algebra for densities as follows:
\begin{enumerate}
\item \textit{Labeling:} $d(f) = s$ if $f$ is a density on $\mathbb{R}^s$.
\item \textit{Combination:} For densities $f$ and $g$ with $d(f) = s$ and $d(g) = t$ and $x \in \mathbb{R}^{s \cup t}$,
\begin{eqnarray}
(f \cdot g)(x) = f(x_s)g(x_t),
\nonumber
\end{eqnarray}
where $x_s$ and $x_t$ denote the subtuples of components of $x$ in $s$ and $t$ respectively.
\item \textit{Projection:} For a density $f$ with $d(f) = s$, $t \subseteq s$, and $x \in \mathbb{R}^s$,
\begin{eqnarray}
(\pi_t(f))(x_t) = \int_{- \infty}^{+ \infty} f(x_t,x_{s-t}) dx_{s -t}.
\nonumber
\end{eqnarray}
\end{enumerate}
It is straightforward to verify the axioms of a valuation algebra for this system. There are no unit elements, since the function $e(x) = 1$ for all $x \in \mathbb{R}^s$ is not finitely integrable, hence not a density.  However, the strong Combination Axiom A5' is satisfied for densities. Further there are null elements $0_s(x) = 0$ for all $x \in \mathbb{R}^s$. Note that the combination operation of densities has no obvious sense in terms of classical probability theory. Still, this valuation algebra is important for local computation with factorisation of a density into a product of conditional densities, similar to the case of probability densities.

This valuation algebra is separative. In fact, define for two densities $f$ and $g$ the relation $f \equiv_\delta g$ if $f(x) > 0 \Leftrightarrow g(x) > 0$. Define for a density $f$ on $\mathbb{R}^s$ the support set $supp(f) = \{x \in \mathbb{R}^s:f(x) > 0\}$. Then two densities are equivalent, exactly if they have the same support sets. So, the equivalence classes $[f]_\delta$ consist of all densities with the same support set $supp(f)$. They form clearly a subsemigroup of the semigroup of all densities. The null elements have empty support sets and each $0_s$ forms by itself a semigroup, in fact already a (trivial) group. Note that the density $f \cdot \pi_t(f)$ has the same support as $f$ and that the semigroup $[f]_\delta$ is cancellative. The semigroup $[f]_\delta$ is therefore embedded into the group $\delta(f)$ of all quotients of densities $h(x)/g(x)$ with support sets equal to $supp(f)$. Here we take the liberty to replace the equivalence classes of pairs $[f,g]$ by the equivalent quotients. The unit of this group $\delta(f)$ is the function $e_f(x) = 1$ if $x \in supp(f)$ and vanishing elsewhere. If $supp(f)$ is not of finite measure, then $e_f$ is not a density. The embedding is by the map $f \mapsto f/e_f = f \cdot f/f$. The inverse of a density $f$ is then the function $f^{-1}(x) = 1/f(x)$ if $x \in supp(f)$ and vanishing elsewhere. This function is in general no more finitely integrable, hence no more a density. So, the valuation algebra of densities is not regular. It is also not cancellative. Density functions provide in some sense the model for general separative valuation algebras.
\end{example}

Here follows an example of a cancellative valuation algebra.

\begin{example} \textit{Gaussian Potentials:}

Gaussian densities are of particular interest in applications. A multivariate Gaussian density over a set $s$  of variables is defined by
\begin{eqnarray}
f(x) = (2\pi)^{-n/2} (det\ \Sigma)^{-1/2}e^{-(1/2)(x - \mu)'\Sigma^{-1}(x - \mu)}.
\nonumber
\end{eqnarray}
Here $\mu$ is a vector in $\mathbb{R}^s$ (see previous example), and $\Sigma$ is a symmetric positive definite matrix in $\mathbb{R}^s \times \mathbb{R}^s$. The vector $\mu$ is the expected value vector of the density and $\Sigma$ the variance-covariance matrix. The matrix $\mathbf{K} = \Sigma^{-1}$ is called the concentration matrix of the density. It is also symmetric and positive definite. A Gaussian density may be represented or determined by the pair $(\mu,\mathbf{K})$. Each such pair with $\mu$ an $s$-vector and $\mathbf{K}$ a symmetric positive definite $s \times s$ matrix determines a Gaussian density. Gaussian densities belong to the valuation algebra of densities defined in the previous example. In fact, they form a subalgebra of the algebra of densities. Labeling, combination and projection can however now  be expressed in terms of the pairs $(\mu,\mathbf{K})$. For this purpose, if $t \supseteq s$ and $\mu$ is an $s$-vector, $\mathbf{K}$ a $s \times s$ matrix let $\mu^{\uparrow t}$ and $\mathbf{K}^{\uparrow t}$ be the vector or matrix obtained by adding to $\mu$ and $\mathbf{K}$ $0$-entries for all indices in $t - s$. Further, if $t \subseteq s$, then let $\mu_t$ and $\mathbf{K}_{t,t}$ be the subvector or submatrix of $\mu$ and $\mathbf{K}$ respectively with components in $t$.

We then define the following operations on pairs $(\mu,\mathbf{K})$:
\begin{enumerate}
\item \textit{Labeling:} $d(\mu,\mathbf{K}) = s$ if $\mu$ is a $s$-vector and $\mathbf{K}$ a $s \times s$ matrix.
\item \textit{Combination:} For pairs $(\mu_1,\mathbf{K}_1)$ and $(\mu_2,\mathbf{K}_2)$ with $d(\mu_1,\mathbf{K}_1) = s$ and $d(\mu_2,\mathbf{K}_2) = t$,
\begin{eqnarray}
(\mu_1,\mathbf{K}_1) \cdot (\mu_2,\mathbf{K}_2) = (\mu,\mathbf{K})
\nonumber
\end{eqnarray}
with
\begin{eqnarray}
\mathbf{K} = \mathbf{K}_1^{\uparrow s \cup t} + \mathbf{K}_2^{\uparrow s \cup t}
\nonumber
\end{eqnarray}
and 
\begin{eqnarray}
\mu = \mathbf{K}^{-1} \left( \mathbf{K}_1^{\uparrow s \cup t} \mu_1^{\uparrow s \cup t} + 
\mathbf{K}_2^{\uparrow s \cup t} \mu_2^{\uparrow s \cup t} \right).
\nonumber
\end{eqnarray}
\item \textit{Projection:} For a pair $(\mu,\mathbf{K})$ with $d(\mu,\mathbf{K}) = s$, $t \subseteq s$,
\begin{eqnarray}
\pi_t(\mu,\mathbf{K}) = (\mu_t,((\mathbf{K}^{-1})_{t,t})^{-1}).
\nonumber
\end{eqnarray}
\end{enumerate}
This is justified by the fact, that the combination of two Gaussian densities as densities results again in a Gaussian density, and so does projection of Gaussian density. We refer to \cite{kohlas03} for more details. Again, the algebra of Gaussian potentials has no unit elements, but satisfies the strong Combination Axiom A5'. For an application of this valuation algebra to linear systems with Gaussian disturbances we refer to \cite{poulykohlas11}.

This valuation algebra is cancellative: Let $(\mu,\mathbf{K})$, $(\mu_1,\mathbf{K}_1)$ and $(\mu_2,\mathbf{K}_2)$ represent Gaussian densities such that $(\mu,\mathbf{K}) \cdot (\mu_1,\mathbf{K}_1) = (\mu,\mathbf{K}) \cdot (\mu_2,\mathbf{K}_2)$. Then $\mathbf{K} + \mathbf{K}_1 = \mathbf{K} + \mathbf{K}_2$, hence $\mathbf{K}_1 = \mathbf{K}_2$ and 
\begin{eqnarray}
(\mathbf{K} + \mathbf{K}_1)^{-1}(\mathbf{K}\mu + \mathbf{K}_1\mu_1) = (\mathbf{K} + \mathbf{K}_2)^{-1}(\mathbf{K}\mu + \mathbf{K}_2\mu_1)
\end{eqnarray}
implies then, in view of the above, that $\mu_1 = \mu_2$. Therefore, the valuation algebra of Gaussian densities, as the algebra of densities, is embedded into the semigroup of the union of groups of quotients of Gaussian densities on $\mathbb{R}^s$ for all finite subsets $s$ of variables. The units of these groups are the functions $e_s(x) = 1$ for all $x \in \mathbb{R}^s$. For instance, conditional Gaussian densities $g(x)/\pi_t(g)(x_t) = (g(x)/e_s(x))(\pi_t(g)(x_t)/e_t(x)$ belong to this semigroup (see the next Section for conditionals in general). 
\end{example}

Next, belief functions provide another example of a separative valuation algebra.

\begin{example} \textit{Belief Functions:}

Here we take for the lattice $(D,\leq)$ any sublattice of partitions of some universe $U$ with a \textit{finite} number of blocks. The usual model in this context considers finite sets of variables with each variable having a finite set of values (like in the case of probability potentials above). This is a special case of our more general frame, which corresponds more to the framework considered in \cite{shafer76}, the original source of belief functions, see also \cite{shafershenoymellouli97,kohlasmonney95}. We summarize some elementary facts about partitions. There is a partial order between partitions defined by $\mathcal{P}_1 \leq \mathcal{P}_2$, if every block of $\mathcal{P}_2$ is contained in a block of $\mathcal{P}_1$, that is $\mathcal{P}_2$ is finer than $\mathcal{P}_1$. The join of two partitions $\mathcal{P}_1$ and $\mathcal{P}_2$ is the partition $\mathcal{P}$ whose blocks are exactly the non-empty intersections of a block of $\mathcal{P}_1$ with one of $\mathcal{P}_2$. The meet is bit more involved, but we do not need to enter into details (see \cite{graetzer78}, but note that there the opposite order is used). It is convenient to associate to a partition $\mathcal{P}$ the set $\Theta_\mathcal{P}$ of its blocks. We call this the frame of the partition, and we define an order between frames by $\Theta_{\mathcal{P}_1} \leq \Theta_{\mathcal{P}_2}$ if and only if $\mathcal{P}_1 \leq \mathcal{P}_2$. The family of frames becomes then a lattice just as the original lattice of partitions. In the following we consider the lattice $(D;\leq)$ of frames rather than the corresponding lattice of partitions. If $\Theta$ and $\Lambda$ are two frames such that $\Theta \leq \Lambda$, then 
we define a map of subsets of $\Theta$ to subsets of $\Lambda$ by
\begin{eqnarray*}
\tau_\Lambda(S) = \{\lambda \in \Lambda:\lambda \subseteq \theta \textrm{ for some}\ \theta \in S\},
\end{eqnarray*}
for $S \subseteq \Theta$. This is called a refining of $\Theta$. In the other direction, we define a map
\begin{eqnarray*}
v_\Theta(T) = \{\theta \in \Theta: \tau_\Lambda(\{\theta\}) \cap T \not= \emptyset\}
\end{eqnarray*}
for $T \subseteq \Lambda$.

Now we are ready to define belief functions on frames (or equivalently on partitions). Consider $(D,\leq)$ to be the lattice of frames associated with the original lattice of partitions. For any frame $\Theta \in D$ consider functions $m$ from the \textit{power set} $2^\Theta$ to nonnegative real numbers and let $\Psi_\Theta$ be the set of all such functions. Define 
\begin{eqnarray*}
\Psi = \bigcup_{\Theta \in D} \Psi_\Theta.
\end{eqnarray*}
This corresponds to basic probability assignments in Dempster-Shafer theory \cite{shafer76}, except that there the functions $m$ are only defined for non-empty subsets and are normalised such that the sum of all $m(S)$ equals one. But as with probability potentials, these side conditions may be neglected for our purposes. We call $m$ mass functions. We define then the following operations:
\begin{enumerate}
\item \textit{Labeling:} $d(m) = \Theta$ if $m$ is defined on frame $\Theta$,
\item \textit{Combination:} If $d(m_1) = \Theta$ and $d(m_2) = \Lambda$, $S \subseteq \Theta \vee \Lambda$, then $m = m_1 \cdot m_2$ is defined by
\begin{eqnarray*}
m(S) = \sum \{m_1(S_1)m_2(S_2):S_1 \subseteq \Theta,S_2 \subseteq \Lambda,v_{\Theta \vee \Lambda}(S_1) \cap v_{\Theta \vee \Lambda}(S_2) = S\}
\end{eqnarray*}
\item \textit{Projection:} If $d(m) = \Lambda$ and $\Theta \leq \Lambda$, and $S \subseteq \Theta$, then $\pi_\Theta(m)$ is defined by
\begin{eqnarray*}
\pi_\Theta(S) = \sum \{m(T);T \subseteq \Lambda,\tau_\Lambda(T) = S\}.
\end{eqnarray*}
\end{enumerate}
The combination is essentially Dempster's rule \cite{shafer76}, except that normalisation is missing. This defines a valuation algebra, see for instance \cite{kohlas03,poulykohlas11,kohlas16}. This algebra has for every frame $\Theta \in D$ a unit element defined by $1_\Theta(\Theta) = 1$ and $1_\Theta(S) = 0$ for any proper subset $S$ of $\Theta$. These unit elements satisfy Axiom A6. It has also null elements $0_\Theta(S) = 0$ for all subsets $S$ of $\Theta$.

To any mass function $m$, we can associate two other set functions by its Moebius transforms,
\begin{eqnarray*}
b(S) = \sum_{T \subseteq S} m(T), \quad q(S) = \sum_{T \supseteq S} m(T).
\end{eqnarray*}
The function $b$ is called \textit{belief function} and $q$ \textit{commonality function}. There is a one-to-one correspondence between mass, belief and commonality functions. In fact, we have \cite{shafer76}
\begin{eqnarray*}
m(S) = \sum_{T \subseteq S} (-1)^{\vert S - T \vert} b(T) = \sum_{T \supseteq S} (-1)^{\vert T - S \vert}.
\end{eqnarray*}
The crucial point is that if $q_1$ and $q_2$ are the commonality functions corresponding to the mass functions $m_1$ and $m_2$ with $d(m_1) = \Theta$ and $d(m_2) = \Lambda$, then the commonality function $q$, corresponding to the combined mass function $m_1 \cdot m_2$ is defined for any subset $S$ of $\Theta \vee \lambda$ by
\begin{eqnarray*}
q(S) = q_1(v_\Theta(S))q_2(v_\Lambda(S)),
\end{eqnarray*}
see \cite{shafer76,kohlas03}.

The valuation algebra of mass functions is separative. In fact we use commonality functions to define a congruence $q_1 \equiv_\delta q_2$ if $q_1(S) > 0 \Leftrightarrow q_2(S) > 0$. If we denote by $\pi_\Lambda(q)$ the commonality function of $\pi_\Lambda(m)$, and by $q_1 \cdot q_2$ the commonality function of $m_1 \cdot m_2$, then it can be seen that $q(S) > 0$ implies $\pi_\Lambda(q)(\tau(S)) > 0$ and therefore $q \cdot \pi_\Lambda(q) \equiv_\delta q$. Define $supp(q) = \{S \subseteq \Theta:q(S) > 0\}$ if $q$ is a commonality function on frame $\Theta$. The subsemigroup of commonality functions $[q]_\delta$ with support $supp(q)$ is clearly cancellative. So, the valuation algebra of mass functions is indeed separative. The groups $\delta(q)$ consist of quotients $q_1(S)/q_2(S)$ if $S \in supp(q)$ and zero otherwise, for two commonality functions with the same support $supp(q)$. The unit of group $\delta(q)$ is the function $f_q(S) =1$ for $S \in supp(q)$ and $f_q(S) = 0$ otherwise. The inverse of commonality $q$ is simply $1/q(S)$ for $S \in supp(q)$ and zero otherwise. Although the quotients $q_1/q_2$ as well as the units $1_q$ are nonnegative functions, its Moebius transforms are no more nonnegative mass functions in general. This means that units and inverses and quotients $q_1/q_2$ in general do not belong to the valuation algebra. Therefore the valuation algebra is not regular, and it is not cancellative either.
\end{example}

Further examples of separative valuation algebras may be found in \cite{poulykohlas11} as well as in \cite{kohlaswilson06}.

%%%%%%%%%%%%%%%%%%%%%%%%%%%%%%%%%%%%%%%%%%%%%%%%%%%%%%%%%%%%

\section{Conditionals} \label{sec:Conditionals}

In probability theory, \textit{conditioning} and \textit{conditional distributions} play an important role as well as independence and \textit{conditional independence}. This applies equally to modeling and to computational purposes. We claim that these concepts are not limited to probability, but concern more generally \textit{information} in a wider context. Therefore, we examine generalisations in this section in the realm of separative valuation algebras, because conditioning presupposes a concept of division.

We assume throughout this section $(\Psi,D;\leq,d,\cdot,\pi)$ to be a \textit{separative valuation algebra}, In addition, we assume either that the valuation algebra has unit elements satisfying axioms A6 or else that the extended combination axiom A5' is satisfied . This guarantees that partial projection in $\Psi^0$ is well defined. Note that the theory developed in this section covers all examples in the previous section.

We define the concept of a conditional, following the pattern of probability distributions. The results presented in this section were already exposed in \cite{kohlas03}, however only in the multvariate setting. The results extend easily to the more general case of lattices $(D;\leq)$ of domains.

\begin{definition} \label{def:Conditional} \textit{Conditional:}
Let $(\Psi,D;\leq,d,\cdot,\pi)$ be a separative valuation algebra. For an element $\phi \in \Psi$, and $y \leq x \leq d(\phi)$, 
\begin{eqnarray}
\phi_{x \vert y} = \pi_x(\phi) \cdot (\pi_y(\phi))^{-1}
\end{eqnarray}
is called a conditional of $\phi$ for $x$ given $y$.
\end{definition}

Note that a conditional $\phi_{x \vert y}$ does, in general, not belong to $\Psi$, but only to $\Psi^0$, except if the valuation algebra is regular. In this case a conditional $\phi_{x \vert y}$ can be projected to all domains $z \leq x$, whereas in general for a conditional $\phi_{x \vert y}$ projections exist only for $z$ such that $y \leq z \leq x$, since $d(\phi_{x \vert y}) = x$ and $\delta(\pi_y(\phi)) \leq \delta(\pi_x(\phi))$. When we consider a conditional $\phi_{x \vert y}$, then we assume always implicitly that $y \leq x \leq d(\phi)$.
. Further, it follows from the definition that
\begin{eqnarray}
\pi_x(\phi) = \phi_{x \vert y} \cdot \pi_y(\phi)
\end{eqnarray}
since $\delta(\pi_y(\phi)) = \delta(\pi_y(\pi_x(\phi))) \leq \delta(\pi_x(\phi))$ by (\ref{eq:DomProj}). For this reason conditionals  $\phi_{x \vert y}$ were also called \textit{continuers} of $\phi$ from $y$ to $x$ in \cite{shafer96}, or we say that $\phi_{x \vert y}$ continues $\phi$ from $y$ to $x$. We have also $\delta(\phi_{x \vert y}) = \delta(\pi_x(\phi) \cdot (\pi_y(\phi))^{-1}) = \delta(\pi_x(\phi)) \vee \delta(\pi_y(\phi)) = \delta(\pi_x(\phi))$, so that  
\begin{eqnarray} \label{eq:DomOfCondOverProj}
\delta(\pi_y(\phi)) \leq \delta(\phi_{x \vert y}).
\end{eqnarray}

Here follow a few elementary results about conditionals.

\begin{lemma} \label{le:PropOfCond}
Let $(\Psi,D;\leq,d,\cdot,\pi)$ be a separative valuation algebra satisfying the additional assumptions about units or the extended combination axiom stated above. Then the following holds:
\begin{enumerate}
\item $\pi_y(\phi_{x \vert y}) = f_{\pi_y(\phi)}$.
\item If $z \leq y \leq x$, then $\phi_{x \vert z} = \phi_{x \vert y} \cdot \phi_{y \vert z}$.
\item If $z \leq y \leq x$, then $\pi_y(\phi_{x \vert z}) = \phi_{y \vert z}$.
\item If $d(\psi) = y \leq x$, then $(\pi_x(\phi) \cdot \psi)_{x \vert y} = \phi_{x \vert y} \cdot f_\psi$.
\item If $z \leq y \leq x$ and $z \leq w \leq x$ then $\pi_w(\phi_{x \vert y} \cdot \phi_{y \vert z}) = \phi_{w \vert z}$.
\end{enumerate}
\end{lemma}

\begin{proof}
1.) By definition, by transitivity of projection and the extended combination axiom (Theorem \ref{th:ExtProjTransComb}), $\pi_y(\phi_{x \vert y} )= \pi_y(\pi_x(\phi) \cdot (\pi_y(\phi))^{-1})= \pi_y(\phi) \cdot (\pi_y(\phi))^{-1} = f_{\pi_y(\phi)}$.

2.) Again by definition $\phi_{x \vert y} \cdot \phi_{y \vert z} = \pi_x(\phi) \cdot (\pi_y(\phi))^{-1} \cdot \pi_y(\phi) \cdot (\pi_z(\phi))^{-1} = \pi_x(\phi) \cdot f_{\pi_y(\phi)} \cdot (\pi_z(\phi))^{-1} = \phi_{x \vert z}$ since $\delta(\pi_y(\phi)) \leq \delta(\pi_x(\phi))$. 

3.) Here we have again, using Theorem \ref{th:ExtProjTransComb}, $\pi_y(\phi_{x \vert z}) = \pi_y(\pi_x(\phi) \cdot (\pi_z(\phi))^{-1}) = \pi_y(\pi_x(\phi) \cdot f_{\pi_y(\phi)} \cdot (\pi_z(\phi))^{-1}) = \pi_y(\pi_x(\phi)) \cdot f_{\pi_y(\phi)} \cdot (\pi_z(\phi))^{-1}
= \pi_y(\pi_x(\phi)) \cdot (\pi_z(\phi))^{-1} = \pi_y(\phi) \cdot (\pi_z(\phi))^{-1} = \phi_{y \vert z}$.

4.) On the one hand, we have $\pi_x(\phi) \cdot \psi = \phi_{x \vert y} \cdot \pi_y(\phi) \cdot \psi$ since $\phi_{x \vert y}$ continues $\phi$ from $y$ to $x$. On the other hand we have also $\pi_x(\phi) \cdot \psi = (\pi_x(\phi) \cdot \psi)_{x \vert y} \cdot \pi_y(\pi_x(\phi) \cdot \psi) = (\pi_x(\phi) \cdot \psi)_{x \vert y} \cdot \pi_y(\phi) \cdot \psi$, again using the continuation property of a conditional. This leads to the equation
\begin{eqnarray}
\phi_{x \vert y} \cdot (\pi_y(\phi) \cdot \psi)= (\pi_x(\phi) \cdot \psi)_{x \vert y} \cdot (\pi_y(\phi) \cdot \psi.)
\nonumber
\end{eqnarray}
Multiplying both sides with the inverse of $\pi_y(\phi) \cdot \psi$, we get
\begin{eqnarray}
\phi_{x \vert y} \cdot f_{\pi_x(\phi) \cdot \psi} 
= (\pi_x(\phi) \cdot \psi)_{x \vert y} \cdot f_{\pi_x(\phi) \cdot \psi}.
\nonumber
\end{eqnarray}
By (\ref{eq:DomOfCondOverProj}) we have  $\delta(\pi_y(\phi) \cdot \psi) \leq \delta((\pi_x(\phi) \cdot \psi)_{x \vert y})$. Then it follows that 
\begin{eqnarray}
(\pi_x(\phi) \cdot \psi)_{x \vert y} = \phi_{x \vert y} \cdot f_{\pi_y(\phi) \cdot \psi} 
= \phi_{x \vert y} \cdot f_{\pi_y(\phi)} \cdot f_\psi = \phi_{x \vert y} \cdot f_\psi
\nonumber
\end{eqnarray}
where the last equality follows again from (\ref{eq:DomOfCondOverProj}).

5.) By item 2 proved above, $\phi_{x \vert y} \cdot \phi_{y \vert z} = \phi_{x \vert z}$ and since $z \leq w \leq x$ it follows from item 3 above that $\pi_w(\phi_{x \vert z}) = \phi_{w \vert z}$.
\end{proof}

If we consider briefly conditionals in the examples of the previous section, then we remark that in the case of probability potentials, density and Gaussian potentials conditionals in  the sense of Definition \ref{def:Conditional} correspond essentially (up to normalisation) to conditional discrete probability distributions, conditional densities and conditional Gaussian densities. Lemma \ref{le:PropOfCond} corresponds in these cases to well-known results of probability theory. In the case of probability potentials conditionals are themselves probability potentials. This is not the case for densities, they belong in general only to the extending semigroup $\Psi^0$. The case of belief functions is less usual.

After these preparations, we are ready to turn to the discussion of compositional operators in separative valuation algebras.

%%%%%%%%%%%%%%%%%%%%%%%%%%%%%%%%%%%%%%%%%%%%%%%%%%%%%%%%%%%%

\section{Compositional Operators} \label{sec:CompOp}

Compositional models have been introduced as an alternative to Bayesian networks in \cite{jirousek97,jirousek11}. Later these models were extended for possibility theory \cite{vejnarova98} and for Dempster-Shafer theory \cite{jirousekvejn07}. Finally it was noted that compositional models can be formed and used, under some conditions, in valuation based systems \cite{jirousekshenoy14,jirousekshenoy15}. In the last two references, an axiomatic system essentially equivalent to regular valuation algebras are assumed, which excludes for instance compositional models of belief functions. We show here that in fact compositional models are possible in \textit{separative} valuation algebras. This allows then, among other instances of valuation algebras like Gaussian densities, also to include belief functions into the framework of compositional models.

Let then $(\Psi,D;\leq,d,\cdot,\pi)$ throughout this section be a separative valuation algebra having units satisfying Axiom A6 or else the extended Combination Axiom S5', so that partial projection in the algebra is well defined. We introduce first the compositional operator and prove then two central theorems (Theorems \ref{th:MainCompTheorem1} and \ref{th:MainCompTheorem2}) underlying the theory of compositional models. This shows then that results obtained in \cite{jirousek97,jirousek11} and \cite{jirousekshenoy14,jirousekshenoy15} apply to separative valuation algebras.  However, these two theorems are not valid for general lattices. The first one holds for modular and the second one for distributive lattices. A lattice is modular, if  $z \leq y$ implies 
\begin{eqnarray*}
y \wedge (x \vee z) = (x \wedge y) \vee z.
\end{eqnarray*}
A lattice is distributive, if 
\begin{eqnarray*}
x \wedge (y \vee z) = (x \wedge y) \vee (x \wedge z).
\end{eqnarray*}
A distributive lattice is also modular \cite{daveypriestley97}.

\begin{definition} \label{def:CompOp}
If $\phi,\psi \in \Psi^0$ with $d(\phi) = x$, $d(\psi) = y$, such that $\pi_{x \wedge y}(\phi$ and $\pi_{x \wedge y}(\psi)$ both exist, and further $\delta(\pi_{x \wedge y}(\psi)) \leq \delta(\pi_{x \wedge y}(\phi))$ , define
\begin{eqnarray}
\phi \rhd \psi = \phi \cdot \psi_{y \vert x \wedge y} = \phi \cdot \psi \cdot (\pi_{x \wedge y}(\psi))^{-1}.
\end{eqnarray}
This is called a composition of $\phi$ and $\psi$ and $\rhd$ is called the compositional operator.
\end{definition}

Note that $\phi\ \rhd \psi$ exists for all elements in  $\Psi = \Psi_0$, if the valuation algebra is \textit{regular} and $\Psi$ is closed under composition. If the valuation algebra is not regular, but only separative, the situation is a bit more involved. We assume then that the domains in $D$ have a least element $\bot$. We call an element $\psi$ of $\Psi_0$ an \textit{abstract density}, if all projections $\pi_x(\psi)$ for $x \leq d(\psi)$ exist, which is equivalent to assume that the projection to the least element $\pi_\bot(\psi)$ exists. Note that all elements of $\Psi$ are then abstract densities. Further in the case of (usual) density functions and Gaussian densities, the elements of $\Psi$ are the only abstract densities. In the case of a general separative valuation algebra it is an open question whether there are elements not in $\Psi$ which are (abstract) densities. Let $\Psi-d$ denote the family of abstract densities in the separative valuation algebra $(\Psi,D;\leq,d,\cdot,\pi)$. Clearly, $\Psi_d$ is closed under projection. We claim that it is also closed under composition.

\begin{theorem} \label{th:Composition}
Let $(\Psi,D;\leq,d,\cdot,\pi)$ be a separative valuation algebra satisfying the additional assumptions about units or extended combination axiom stated above. If $\phi$ and $\psi$ are abstract densities such that $\phi \rhd \psi$ is defined, then $\phi \rhd \psi$ is an abstract density. 
\end{theorem}

\begin{proof}
Assume $d(\phi) = x$ and $d(\psi) = y$ and that $\phi \rhd \psi$ is defined. We show that the projection $\pi_\bot(\phi \cdot \psi_{y \vert x \wedge y})$ exists. In fact,
\begin{eqnarray*}
\lefteqn{\pi_{x \wedge y}(\phi \rhd \psi) = \pi_{x \wedge y}(\phi \cdot \psi \cdot (\pi_{x \wedge y}(\psi))^{-1}}Ê\\
&&=\pi_{x \wedge y}(\pi_x(\phi \cdot \psi \cdot (\pi_{x \wedge y}(\psi))^{-1})) = \pi_{x \wedge y}(\phi \cdot \pi_{x \wedge y}(\psi \cdot (\pi_{x \wedge y}(\psi))^{-1}) \\
&&= \pi_{x \wedge y}(\phi) \cdot \pi_{x \wedge y}(\psi) \cdot (\pi_{x \wedge y}(\psi))^{-1} = \pi_{x \wedge y}(\phi) \cdot f_{\pi_{x \wedge y}(\psi)} =
\pi_{x \wedge y}(\phi).
\end{eqnarray*}
Since $\phi$ is an abstract density, $\pi_\bot(\phi)$ exists and $\pi_\bot(\phi \rhd \psi) = \pi_\bot(\pi_{x \wedge y}(\phi \rhd \psi)) = \pi_\bot(\phi)$ and therefore $\pi_\bot(\phi \rhd \psi)$ exists too, hence $\phi \rhd \psi$ is an abstract density.
\end{proof}

The following theorems give the main properties of composition. It is a mathematically rigorous generalisation of the main theorems in \cite{jirousekshenoy14,jirousekshenoy15} to separative valuation algebras. 

\begin{theorem} \label{th:MainCompTheorem1}
Let $(\Psi,D;\leq,d,\cdot,\pi)$ be a separative valuation algebra satisfying the additional assumptions about units or the extended combination axiom stated above and $(D;\leq)$ a modular lattice. Let $\phi,\psi$ be abstract denisties with $d(\phi) = x$, $d(\psi) = y$, and $\delta(\pi_{x \wedge y}(\psi)) \leq \delta(\pi_{x \wedge y}(\phi))$. Then,
\begin{enumerate}
\item $d(\phi \rhd \psi) = x \vee y$.
\item $\pi_x(\phi \rhd \psi) = \phi$.
\item If $y \leq x$, then $\phi \rhd \psi = \phi$.
\item $\pi_{x \wedge y}(\phi) = \pi_{x \wedge y}(\psi)$ implies $\phi \rhd \psi = \psi \rhd \phi$.
\item If $x \wedge y \leq z \leq y$ then $(\phi \cdot \pi_z(\psi)) \rhd \psi = \phi \cdot \psi$.
\item If $x \wedge y \leq z \leq y$, then $(\phi \rhd \pi_z(\psi)) \rhd \psi = \phi \rhd \psi$.
\end{enumerate}
\end{theorem}

\begin{proof}
The proof depends on the generalised valuation algebra axioms for separative valuation algebras, especially the extended combination axiom  (Theorem \ref{th:ExtProjTransComb}). 

1.) is a simple consequence of the (generalised) Labeling Axiom.

2.) By the Combination Axiom, we have
\begin{eqnarray*}
\lefteqn{\pi_x(\phi \cdot \psi \cdot (\pi_{x \wedge y}(\psi))^{-1})}Ê\\ 
&&=\phi \cdot \pi_{x \wedge y}(\psi \cdot (\pi_{x \wedge y}(\psi))^{-1}) \\
&&= \phi \cdot \pi_{x \wedge y}(\psi) \cdot (\pi_{x \wedge y}(\psi))^{-1} \\
&&= \phi \cdot f_{\pi_{x \wedge y}(\psi)} = \phi,
\end{eqnarray*}
since $\delta(\pi_{x \wedge y}(\psi)) \leq \delta(\pi_{x \wedge y}(\phi)) \leq \delta(\phi)$.

3.) If $y \leq x$ then $x \wedge y = y$ and therefore $\phi \rhd \psi = \phi \cdot f_{\pi_{y}(\psi)} = \phi$ as before.

4.) If $\pi_{x \wedge y}(\phi) = \pi_{x \wedge y}(\psi)$, then $\phi \rhd \psi = \phi \cdot \psi \cdot (\pi_{x \wedge y}(\psi))^{-1} = \psi \cdot \phi \cdot (\pi_{x \wedge y}(\phi))^{-1} = \psi \rhd \phi$ since $\delta(\pi_{x \wedge y}(\psi)) = \delta(\pi_{x \wedge y}(\phi))$.

5.) If $z \leq y$ then by the modular law in the modular lattice $(D;\leq)$ we obtain that $(x \vee z) \wedge y = (x \wedge y) \vee z$ and from $x \wedge y \leq z$ it follows that $(x \vee z) \wedge y = z$. Note that $\delta(\pi_z(\phi \cdot \pi_z(\psi))) = \delta(\pi_z(\phi) \cdot \pi_z(\psi))
\geq \delta(\pi_z(\psi))$ so that $(\phi \cdot \pi_z(\psi)) \rhd \psi$ is well defined. 

We have
\begin{eqnarray*}
\lefteqn{(\phi \cdot \pi_z(\psi)) \rhd \psi}Ê\\
&&= \phi \cdot \pi_z(\psi) \cdot \psi \cdot (\pi_{(x \vee z) \wedge y}(\psi))^{-1} \\
&&= \phi \cdot \pi_z(\psi) \cdot \psi \cdot (\pi_z(\psi))^{-1} \\
&&= \phi \cdot \psi \cdot f_{\pi_z(\psi)}.
\end{eqnarray*}
But $\delta(\pi_z(\psi)) \leq \delta(\psi)$ implies then that $(\phi \rhd \pi_z(\psi)) \rhd \psi = \phi \cdot \psi$.

6.) As before, we have $(x \vee z) \wedge y = z$. Further $x \wedge y \leq z \leq y$ implies $x \wedge y \leq x \wedge z \leq x \wedge y$, hence $x \wedge z = x \wedge y$. This implies that $\phi \rhd \pi_z(\psi)$ is a density and all compositions are well defined. This time we have
\begin{eqnarray*}
\lefteqn{(\phi \rhd \pi_z(\psi)) \rhd \psi}Ê\\
&&= (\phi \cdot \pi_z(\psi) \cdot (\pi_{x \wedge z}(\psi))^{-1}) \cdot \psi 
\cdot (\pi_{(x \vee z) \wedge y}(\psi))^{-1} \\
&&= (\phi \cdot \pi_z(\psi) \cdot (\pi_{x \wedge z}(\psi))^{-1}) \cdot \psi 
\cdot (\pi_z(\psi))^{-1} \\
&&= \phi \cdot \psi \cdot (\pi_{x \wedge y}(\psi))^{-1} \cdot f_{\pi_z(\psi)} \\
&&= \phi \rhd \psi.
\end{eqnarray*}

This concludes the proof of the theorem.
\end{proof}

As a complement, note that if $\phi \rhd \psi = \psi \rhd \phi$, 
\begin{eqnarray}
\phi \cdot \psi \cdot \pi_{x \wedge y}(\phi) = \phi \cdot \psi \cdot \pi_{x \wedge y}(\psi).
\nonumber
\end{eqnarray}
So, in this case we have $\pi_{x \wedge y}(\phi) = \pi_{x \wedge y}(\psi)$, if the valuation algebra is \textit{cancellative}; in this case commutativity of composition implies consistency of the valuations involved.

For distributive lattices $(D;\leq)$ stronger results are possible.

\begin{theorem} \label{th:MainCompTheorem2}
Let $(\Psi,D;\leq,d,\cdot,\pi)$ be a separative valuation algebra satisfying the additional assumptions about units or the extended combination axiom stated above and $(D;\leq)$ a distributive lattice. Let $\phi,\psi$ be abstract densities with $d(\phi) = x$, $d(\psi) = y$, and $\delta(\pi_{x \wedge y}(\psi)) \leq \delta(\pi_{x \wedge y}(\phi))$. Then,
\begin{enumerate}
\item If $x \geq y \wedge z$ and $\tau$ a density with $d(\tau) = z$, $\delta(\pi_{x \wedge z}(\tau)) \leq \delta(\pi_{x \wedge z}(\phi))$, then $(\phi \rhd \psi) \rhd \tau = (\phi \rhd \tau) \rhd \psi$.
\item If $x \wedge y \leq z \leq x \vee y$ then $\pi_z(\phi \rhd \psi) = \pi_{x \wedge z}(\phi) \rhd \pi_{y \wedge z}(\psi)$.
\item If $x \geq y \wedge z$ and and $\tau$ a density with $d(\tau) = z$, $\delta(\pi_{y \wedge z}(\tau)) \leq \delta(\pi_{y \wedge z}(\psi))$ then $(\phi \rhd \psi) \rhd \tau = \phi \rhd (\psi \rhd \tau)$. 
\item If $y \geq x \wedge z$ and and $\tau$ a density with $d(\tau) = z$, $\delta(\pi_{x \wedge z}(\tau)) \leq \delta(\pi_{x \wedge z}(\phi))$ and $\delta(\pi_{y \wedge z}(\tau) \leq \delta(\pi_{y \wedge z}(\psi)$, then $(\phi\ \rhd \psi) \rhd \tau = \phi\ \rhd (\psi \rhd \tau)$.
\end{enumerate}
\end{theorem}

\begin{proof}
1.) The assumptions guarantee that $\phi\ \rhd \psi$ as well as $\phi\ \rhd \tau$ are densities. If $x \geq y \wedge z$, then by the distributivity of the lattice $(D;\leq)$ it follows that $(x \vee y) \wedge z = (x \wedge z) \vee (y \wedge z) = x \wedge z$ and similarly $(x \vee z) \wedge y = x \wedge y$. By item 2 of Theorem \ref{th:MainCompTheorem1}  we have $\pi_{x \wedge z}(\phi \rhd \psi) = \pi_{x \wedge z}(\phi)$, hence $\delta(\pi_{x \wedge z}(\phi \rhd \psi)) \leq \delta(\pi_{x \wedge z}(\tau))$ and simiarly, $\delta(\pi_{x \wedge y}(\phi \rhd \tau)) \leq \delta(\pi_{x \wedge y}(\psi))$. So, $(\phi \rhd \psi) \rhd \tau$ and $\phi \rhd \tau) \rhd \psi$ are well defined. Thus, we have further
\begin{eqnarray*}
\lefteqn{(\phi \rhd \psi) \rhd \tau} \\
 &&= (\phi \cdot \psi \cdot (\pi_{x \wedge y}(\psi))^{-1}) \cdot \tau \cdot (\pi_{(x \vee y) \wedge z}(\tau))^{-1} \\
&&= (\phi \cdot \tau (\pi_{x \wedge z}(\tau))^{-1})  \cdot \psi \cdot (\pi_{(x \vee z) \wedge y}(\psi))^{-1} \\
&&=(\phi \rhd \tau) \rhd \psi.
\end{eqnarray*}

2.) If $x \wedge y \leq z \leq x \vee y$ then we have also $x \wedge y \leq y \wedge z \leq y$ and we can use item 6 of Theorem \ref{th:MainCompTheorem1}, to obtain
\begin{eqnarray}
\pi_{x \vee z}(\phi \rhd \psi) = \pi_{x \vee z}((\phi \rhd \pi_{y \wedge z}(\psi) \rhd \psi).
\nonumber
\end{eqnarray}
Since $x \wedge y \leq z \leq x \vee y$ implies further that $x \vee (y \wedge z) = (x \vee y) \wedge (x \vee z) = x \vee z$ (distributivity) it follows from item 2 of  Theorem \ref{th:MainCompTheorem1} that 
\begin{eqnarray}
\pi_{x \vee z}(\phi \rhd \psi) = \phi \rhd \pi_{y \wedge z}(\psi).
\nonumber
\end{eqnarray}
Note that $\pi_{x \wedge y}(\phi) \rhd \phi = \pi_{x \wedge y}(\phi) \cdot \phi \cdot (\pi_{x \wedge y}(\phi))^{-1} = \phi$. Therefore, using item 1 just proved above, we have
\begin{eqnarray}
\pi_{x \vee z}(\phi \rhd \psi) = (\pi_{x \wedge y}(\phi) \rhd \phi) \rhd \pi_{y \wedge z}(\psi) = (\pi_{x \wedge y}(\phi) 
\rhd \pi_{y \wedge z}(\psi)) \rhd \phi.
\nonumber
\end{eqnarray}

Next we compute $\pi_z(\phi \rhd \psi)$ in the same way, again using properties 6 and 2 of Theorem \ref{th:MainCompTheorem1} and property 1 of the present theorem:
\begin{eqnarray*}
\lefteqn{\pi_z(\phi \rhd \psi) = \pi_z(\pi_{x \vee z}(\phi \rhd \psi))}Ê\\
 &&= \pi_z((\pi_{x \wedge y}(\phi) \rhd \pi_{y \wedge z}(\psi)) \rhd \phi) \\
&&= \pi_z(((\pi_{x \wedge y}(\phi) \rhd \pi_{y \wedge z}(\psi)) \rhd \pi_{x \wedge z}(\phi)) \rhd \phi) \\
&&= (\pi_{x \wedge y}(\phi) \rhd \pi_{y \wedge z}(\psi)) \rhd \pi_{x \wedge z}(\phi) \\
&&= (\pi_{x \wedge y}(\phi) \rhd  \pi_{x \wedge z}(\phi)) \rhd \pi_{y \wedge z}(\psi)) \\
&&= \pi_{x \wedge z}(\phi) \rhd \pi_{y \wedge z}(\psi).
\end{eqnarray*}
We leave it to the reader to verify the conditions for the application of the the properties of Theorem \ref{th:MainCompTheorem1} and of item 1 of the present theorem.

3.) All compositions occurring here are well defined, as can be verified as above in item 1. From the definition of composition, we have
\begin{eqnarray}
\phi \rhd(\psi \rhd \tau) &=&\phi \cdot (\psi \rhd \tau) \cdot (\pi_{x \wedge (y \vee z)}(\psi \rhd \tau))^{-1}
\end{eqnarray}
Here we may apply item 2 of the present theorem for the last term, so that
\begin{eqnarray}
\phi \rhd(\psi \rhd \tau) &=&\phi \cdot (\psi \rhd \tau) \cdot (\pi_{x \wedge y}(\psi) \rhd \pi_{x \wedge z}(\tau))^{-1}.
\nonumber
\end{eqnarray}
Note that under the assumption $x \geq y \wedge z$ we have $(x \wedge y) \wedge (y \wedge z) = y \wedge z$. Therefore, we have $\pi_{x \wedge y}(\psi) \rhd \pi_{x \wedge z}(\tau) = \pi_{x \wedge y}(\psi) \cdot \pi_{x \wedge z}(\tau) \cdot (\pi_{y \wedge z}(\tau))^{-1}$.
This allows us to deduce
\begin{eqnarray*}
\lefteqn{\phi \rhd(\psi \rhd \tau)}Ê\\
&&= \phi \cdot \psi \cdot \tau \cdot (\pi_{y \wedge z}(\tau))^{-1} \cdot (\pi_{x \wedge y}(\psi) \rhd \pi_{x \wedge z}(\tau))^{-1} \\
&&= \phi \cdot \psi \cdot \tau \cdot (\pi_{y \wedge z}(\tau))^{-1} \cdot (\pi_{x \wedge y}(\psi))^{-1} \cdot (\pi_{x \wedge z}(\tau))^{-1} \cdot \pi_{y \wedge z}(\tau) \\
&&= (\phi \rhd  \psi) \cdot \tau \cdot (\pi_{(x \vee y) \wedge z}(\tau))^{-1} \\
&&= (\phi \rhd  \psi) \rhd \tau.
\end{eqnarray*}

4.) Again, all compositions are well defined. Note that from $y \geq x \wedge z$ it follows that $(x \vee y) \wedge z = (x \wedge z) \vee (y \wedge z) = y \wedge z$ and $x \wedge (y \vee z) = (x \wedge y) \vee (x \wedge z) = x \wedge y$. So, using the definition of the compositional operator, we obtain
\begin{eqnarray}
(\phi \rhd \psi) \rhd \tau &=& (\phi \cdot \psi \cdot (\pi_{x \wedge y}(\psi))^{-1}) \cdot \tau \cdot (\pi_{y \wedge z}(\tau))^{-1}
\nonumber
\end{eqnarray}
On the other hand, we have
\begin{eqnarray}
\phi \rhd (\psi \rhd \tau) = \phi \cdot (\psi \rhd \tau) \cdot (\pi_{x \wedge y}(\psi \rhd \tau))^{-1}.
\nonumber
\end{eqnarray}
From item 2 of Theorem \ref{th:MainCompTheorem1} we get, using $\pi_{x \wedge y}(\psi\ \rhd \tau) = \pi_{x \wedge y}(\pi_y(\psi\ \rhd \tau))$,
\begin{eqnarray*}
\lefteqn{\phi \rhd (\psi \rhd \tau)}Ê\\
&&= \phi \cdot (\psi \rhd \tau) \cdot (\pi_{x \wedge y}(\psi))^{-1} \\
&&= \phi \cdot \psi \cdot \tau \cdot (\pi_{y \wedge z}(\tau))^{-1} \cdot (\pi_{x \wedge y}(\psi))^{-1}.
\end{eqnarray*}
The right hand side here is equal to the one above for $(\phi \rhd \psi) \rhd \tau$. So, we have $(\phi \rhd \psi) \rhd \tau = \phi \rhd (\psi \rhd \tau)$.
\end{proof}

At least for distributive lattices $(D;\leq)$ and a few other conditions a separative valuation algebra allows for compositional models, and the results in \cite{jirousekshenoy14,jirousekshenoy15} carry over to this more general case. Only this extension makes compositional modelling for belief functions, Gaussian densities and density functions in general available, since the corresponding valuation algebras are not regular, but but only separative or cancellative respectively.

%%%%%%%%%%%%%%%%%%%%%%%%%%%%%%%%%%%%%%%%%%%%%%%%%%%%%%%%%%%%

\section{Conclusion}

This paper establishes the different structures of valuation algebras, allowing for inverses, conditionals and compositional modelling.

\bibliography{text}

\begin{thebibliography}{}

\bibitem[\protect\citename{Clifford \& Preston, }1967]{cliffordpreston67}
{\sc Clifford, A.~H., \& Preston, G.~B.} 1967.
\newblock {\em Algebraic Theory of Semigroups}.
\newblock Providence, Rhode Island: American Mathematical Society.

\bibitem[\protect\citename{Davey \& Priestley, }1990]{daveypriestley97}
{\sc Davey, B.A., \& Priestley, H.A.} 1990.
\newblock {\em Introduction to Lattices and Order}.
\newblock Cambridge University Press.

\bibitem[\protect\citename{Gr\"atzer, }1978]{graetzer78}
{\sc Gr\"atzer, G.} 1978.
\newblock {\em General Lattice Theory}.
\newblock Academic Press.

\bibitem[\protect\citename{Hewitt \& Zuckerman, }1956]{hewittzucker56}
{\sc Hewitt, E., \& Zuckerman, H.S.} 1956.
\newblock The $l_1$ algebra of a commutative semigroup.
\newblock {\em Amer. Math. Soc.}, {\bf 83}, 70--97.

\bibitem[\protect\citename{Jirousek \& Daniel, }2007]{jirousekvejn07}
{\sc Jirousek, R, J.~Vejnarova, \& Daniel, M.} 2007.
\newblock Compositional Models of Belief Functions.
\newblock {\em Pages  243--252 of:} {\sc De~Cooman~G., J.~Vejnarova, \&
  Zaffalon, M.} (eds), {\em Proc. of the 5-th Int. Symp. on Imprecise
  Probabilities and Their Applications (ISIPTA-07}.
\newblock UAI.
\newblock Charles University Press.

\bibitem[\protect\citename{Jirousek, }1997]{jirousek97}
{\sc Jirousek, R.} 1997.
\newblock Composition of Probability Measures on Finite Spaces.
\newblock {\em Pages  274--281 of:} {\sc Geiger, D., \& Shenoy, P.} (eds), {\em
  Uncertainty in Artificial Intelligence}.
\newblock UAI.
\newblock Morgan Kaufmann.

\bibitem[\protect\citename{Jirousek, }2011]{jirousek11}
{\sc Jirousek, R.} 2011.
\newblock Foundations of Compositional Model Theory.
\newblock {\em Int. J. of General Systems}, {\bf 40}, 623--678.

\bibitem[\protect\citename{Jirousek \& Shenoy, }2014]{jirousekshenoy14}
{\sc Jirousek, R., \& Shenoy, P.} 2014.
\newblock Compositional Models in Valuation Based Systems.
\newblock {\em Int. J. of Approximate Reasoning}, {\bf 55}, 277--293.

\bibitem[\protect\citename{Jirousek \& Shenoy, }2015]{jirousekshenoy15}
{\sc Jirousek, R., \& Shenoy, P.} 2015.
\newblock Causal Compositional Models in Valuation Based Systems with Examples
  in Specific Theories.
\newblock {\em Int. J. of Approximate Reasoning}.

\bibitem[\protect\citename{Kohlas, }2003]{kohlas03}
{\sc Kohlas, J.} 2003.
\newblock {\em Information Algebras: Generic Structures for Inference}.
\newblock Springer-Verlag.

\bibitem[\protect\citename{Kohlas, }2016]{kohlas16}
{\sc Kohlas, J.} 2016.
\newblock {\em Algebraic Structure of Information}.
\newblock
  \url{http://diuf.unifr.ch/drupal/tns/sites/diuf.unifr.ch.drupal.tns/files/AlgStructOfInf_0.pdf)}.

\bibitem[\protect\citename{Kohlas \& Monney, }1995]{kohlasmonney95}
{\sc Kohlas, J., \& Monney, {P.A.}} 1995.
\newblock {\em A Mathematical Theory of Hints. An Approach to the
  Dempster-Shafer Theory of Evidence}.
\newblock Lecture Notes in Economics and Mathematical Systems, vol. 425.
\newblock Springer.

\bibitem[\protect\citename{Kohlas \& Wilson, }2006]{kohlaswilson06}
{\sc Kohlas, J., \& Wilson, N.} 2006.
\newblock {\em Exact and Approximate Local Computation in Semiring Induced
  Valuation Algebras}.
\newblock Tech. rept. 06-06. Department of Informatics, University of Fribourg.

\bibitem[\protect\citename{Kohlas \& Wilson, }2008]{kohlaswilson08}
{\sc Kohlas, J., \& Wilson, N.} 2008.
\newblock Semiring Induced Valuation Algebras: Exact and Approximate Local
  Computation Algorithms.
\newblock {\em Artif. Intell.}, {\bf 172}(11), 1360--1399.

\bibitem[\protect\citename{Lauritzen \& Jensen, }1997]{lauritzenjensen97}
{\sc Lauritzen, S.~L., \& Jensen, F.~V.} 1997.
\newblock Local Computation with Valuations from a Commutative Semigroup.
\newblock {\em Ann. Math. Artif. Intell.}, {\bf 21}(1), 51--69.

\bibitem[\protect\citename{Pouly \& Kohlas, }2011]{poulykohlas11}
{\sc Pouly, M., \& Kohlas, J.} 2011.
\newblock {\em Generic Inference. A Unified Theory for Automated Reasoning}.
\newblock Wiley, Hoboken, New Jersey.

\bibitem[\protect\citename{Shafer, }1976]{shafer76}
{\sc Shafer, G.} 1976.
\newblock {\em A Mathematical Theory of Evidence}.
\newblock Princeton University Press.

\bibitem[\protect\citename{Shafer, }1996]{shafer96}
{\sc Shafer, G.} 1996.
\newblock {\em Probabilistic Expert Systems}.
\newblock CBMS-NSF Regional Conference Series in Applied Mathematics, no. ~67.
\newblock Philadelphia, PA: SIAM.

\bibitem[\protect\citename{Shafer {\em et~al.\ }\relax,
  }1987]{shafershenoymellouli97}
{\sc Shafer, G., Shenoy, P.P., \& Mellouli, K.} 1987.
\newblock Propagating Belief FUnctions in Qualitative Markov Trees.
\newblock {\em Int.\ J.\ of Approximate Reasoning}, {\bf 1}(4), 349--400.

\bibitem[\protect\citename{Shenoy \& Shafer, }1990]{shenoyshafer90}
{\sc Shenoy, P.~P., \& Shafer, G.} 1990.
\newblock Axioms for probability and belief-function proagation.
\newblock {\em Pages  169--198 of:} {\sc Shachter, Ross~D., Levitt, Tod~S.,
  Kanal, Laveen~N., \& Lemmer, John~F.} (eds), {\em Uncertainty in Artificial
  Intelligence 4}.
\newblock Machine intelligence and pattern recognition, vol. 9.
\newblock Amsterdam: Elsevier.

\bibitem[\protect\citename{Shenoy, }1994]{shenoy94c}
{\sc Shenoy, P.P.} 1994.
\newblock Conditional Independence in Valuation-based Systems.
\newblock {\em International Journal of Approximate Reasoning}, {\bf 10},
  203--234.

\bibitem[\protect\citename{Vejnarova, }1998]{vejnarova98}
{\sc Vejnarova, J.} 1998.
\newblock Composition of Possibility Measures on Finite Spaces: Preliminary
  Results.
\newblock {\em Pages  25--30 of:} {\sc Bouchon-Meunier, B., \& Yager, R.R.}
  (eds), {\em Proc. of the 7-th Int. Conf. on Information Processing and
  Management of Uncertainty in Knowledge-Based System (IPMU-98}.
\newblock UAI.

\end{thebibliography}
\bibliographystyle{authordate3}

%%%%%%%%%%%%%%%%%%%%%%%%%%%%%%%%%%%%%%%%%%%%%%%%%%%%%%%%%%%%%%%%%%%%%%%%%%%
%%%%%%%%%%%%%%%%%%%%%%%%%%%%%%%%%%%%%%%%%%%%%%%%%%%%%%%%%%%%%%%%%%%%%%%%%%%

\end{document}